\documentclass[sigconf,nonacm]{acmart}
\usepackage{subcaption}
\usepackage{multirow}
\AtBeginDocument{%
  }

\acmConference[Conference acronym 'XX]{Make sure to enter the correct
  conference title from your rights confirmation email}{June 03--05,
  2018}{Woodstock, NY}
\acmISBN{978-1-4503-XXXX-X/2018/06}

\usepackage{algorithm}
\usepackage{algpseudocode}

\usepackage{gensymb}

\definecolor{poyColor}{RGB}{1, 115, 113}  

\definecolor{yyc_color}{RGB}{150, 45, 45}

\begin{document}

\title{Nipping the Butterfly Effect in the Bud: Self-Output Fine-Tuning for Autoregressive Weather Prediction}



\author{Yun-Ye\ Cai}
\affiliation{%
    \institution{National Taiwan University}
    \city{Taipei}
    \country{Taiwan}
}
\email{r12922104@csie.ntu.edu.tw}

\author{Hsuan-Tien\ Lin}
\affiliation{%
    \institution{National Taiwan University}
    \city{Taipei}
    \country{Taiwan}
}
\email{htlin@csie.ntu.edu.tw}

\begin{abstract}

Long-horizon weather forecasting is a fundamental challenge in atmospheric science, for which autoregressive Deep Learning Weather Prediction (DLWP) has emerged as the primary paradigm. Although the autoregressive pipeline is highly scalable and flexible, its prediction errors grow rapidly over long forecasting horizons. In this work, we study this error growth phenomenon from both theoretical and empirical perspectives. Our analysis reveals that the growth is driven by a feedback loop between output errors and input distribution shifts. Specifically, the autoregressive process amplifies small initial output errors, which progressively corrupt subsequent input distributions, echoing the butterfly effect in atmospheric science and ultimately deteriorating forecasting accuracy over longer horizons.
Furthermore, we show that this distributional shift originates at the earliest stage of inference, with out-of-distribution signatures detectable as early as the first autoregressive step. To mitigate this issue, we propose \textbf{Self-Output Fine-Tuning (SOFT)}, a plug-and-play strategy that leverages the model's own one-step predictions to calibrate the biased input distribution encountered at the first step. Extensive experiments demonstrate that, despite its simplicity, SOFT achieves state-of-the-art performance on long-horizon forecasting tasks and substantially reduces both prediction errors and distributional discrepancy.
The success of SOFT highlights the importance of reexamining the fundamental pipeline of deep learning weather prediction,
representing a critical pipeline advance for atmospheric science.

\end{abstract}

\begin{teaserfigure}
  \centering
  \begin{subfigure}[htb]{0.485\textwidth}
    \centering
    \includegraphics[width=\textwidth,]{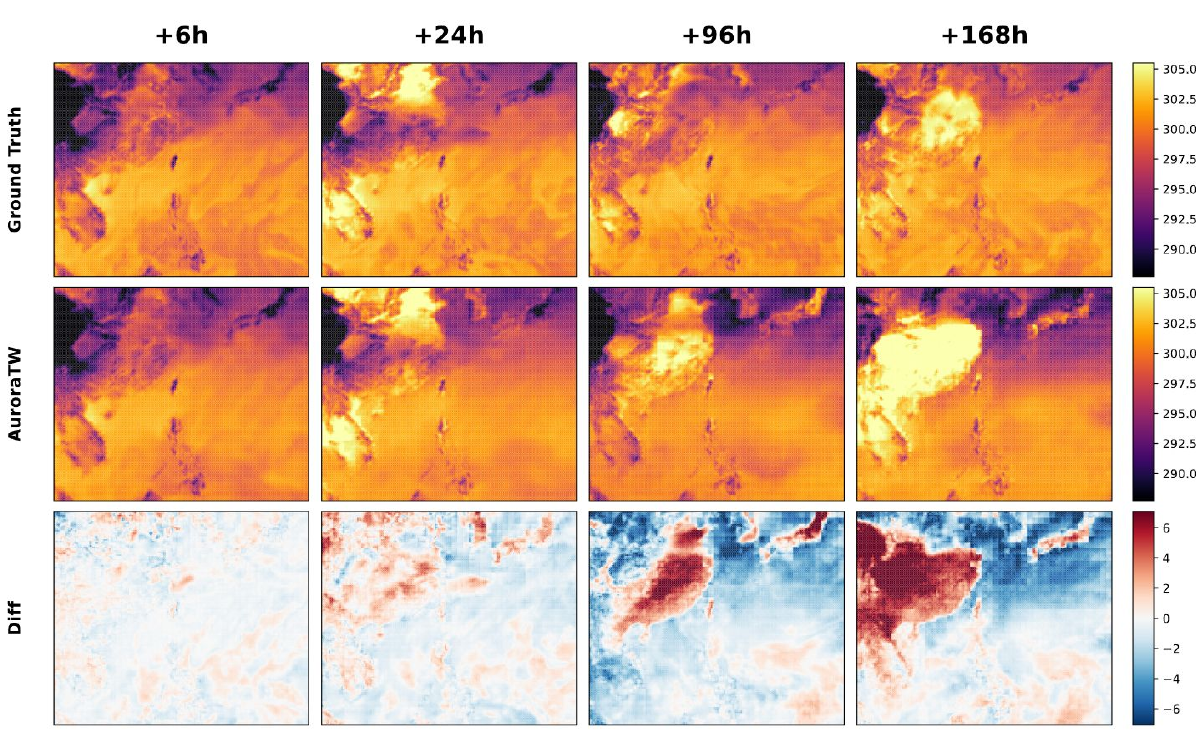}
    \caption{
      Visualization of error growth in the AuroraTW autoregressive baseline.
      As the forecasting horizon expands (from +6h to +168h), prediction errors (row 3) compound, leading to significant deviations from the ground truth.
    }
    \label{fig:preliminary_error_grow}
  \end{subfigure}
  \hfill
  \begin{subfigure}[htb]{0.485\textwidth}
    \centering
    \includegraphics[width=\textwidth,]{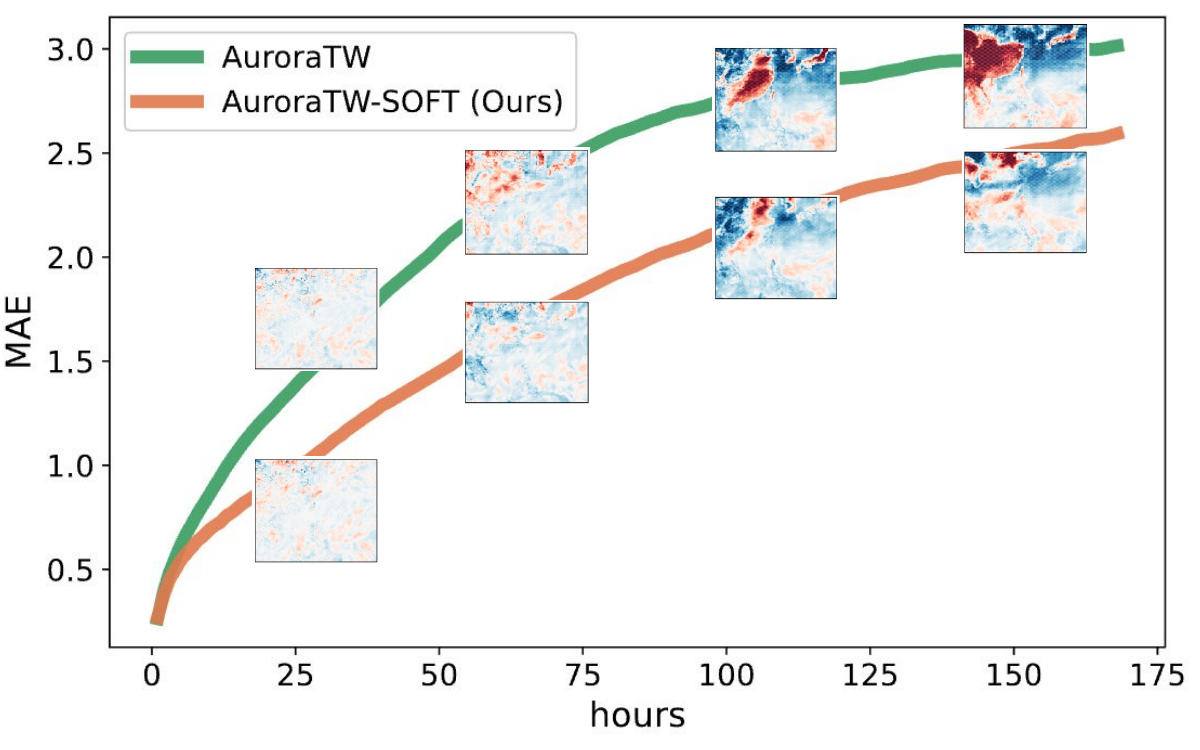}
    \caption{
      Quantitative impact of error growth on 2 meter temperature.
      While the baseline (green) degrades rapidly, our proposed SOFT method (orange) maintains a lower error trajectory over the 168 hour horizon.
    }
    \label{fig:preliminary_curve}
  \end{subfigure}
  \Description{teaser}
  \caption{
    \textbf{Visual and quantitative demonstration of autoregressive error growth and its mitigation via SOFT.}
    We (a) identify that standard autoregressive rollouts induce error growing, and (b) propose Self-Output Fine-Tuning (SOFT) to significantly reduce long-term error.
  }
  \label{fig:overall_figure}
\end{teaserfigure}


\settopmatter{printacmref=false}
\setcopyright{none}
\renewcommand\footnotetextcopyrightpermission[1]{}
\pagestyle{plain}

\maketitle


\section{Introduction}

High-quality meteorological modeling and forecasting play a crucial role in scientific research and societal planning \cite{brunet2023advancing, coughlan2015forecast, meehl2000introduction}. In recent years, deep learning weather prediction (DLWP) models \cite{bi2023accurate, bodnar2025foundation, lam2023graphcast, pathak2022fourcastnet, nguyen2023climax, nguyen2024scaling} achieve competitive performance compared to traditional numerical weather prediction (NWP) methods while offering substantially lower computational cost. Fourcastnet \cite{pathak2022fourcastnet}, Pangu \cite{bi2023accurate}, GraphCast \cite{lam2023graphcast}, Fuxi \cite{chen2023fuxi}, FengWu \cite{chen2023fengwu}, Stormer \cite{nguyen2024scaling} and Baguan \cite{niu2025utilizing} demonstrate the superior accuracy of deep learning models in medium-range forecasting compared to traditional operational systems. In addition, ClimaX \cite{nguyen2023climax}, WeatherGFM \cite{zhao2024weathergfmlearningweathergeneralist}, and Aurora\cite{bodnar2025foundation} have been designed as foundation models, showing broad applicability across various downstream weather tasks.

State-of-the-art DLWP systems largely rely on autoregressive forecasting strategies, which are favored because they allow flexible rollout durations during inference. Autoregressive pipelines perform well in nowcasting ($\le$ 6 hours) and short-range forecasting (1 $\sim$ 3 days); however, prediction errors grow dramatically over medium-range (3 $\sim$ 10 days) and long-range ( $ > $ 10 days) horizons, as illustrated in Figure~\ref{fig:overall_figure}. This phenomenon, known as error accumulation \cite{nguyen2023climax, lam2023graphcast} or cumulative iteration error \cite{bi2023accurate}, poses a critical challenge. Since accurate longer-horizon forecasts are essential for understanding large-scale atmospheric dynamics and supporting climate-related decision-making, mitigating this degradation is crucial for the practical deployment of DLWP systems.

Although various empirical strategies, such as multi-step rollout tuning~\cite{lam2023graphcast, bodnar2025foundation}, temporal hierarchical modeling~\cite{bi2023accurate, chen2023fuxi}, and replay buffers~\cite{bodnar2025foundation}, have been proposed to mitigate this growing error phenomenon,
they generally incur much higher computational costs or rougher temporal resolution.
Moreover, a clear understanding of the root cause of the phenomenon remains limited.
In this work, we argue that error growth in autoregressive strategies fundamentally reflects a \textbf{distributional robustness problem}. Specifically, we hypothesize that this error growth is mainly driven by a feedback loop between output errors and input distribution shifts. The loop originates from the autoregressive pipeline, where output prediction errors propagate as biases in the inputs for subsequent steps. As these biases accumulate, the input sequence progressively drifts away from the original training data distribution. This drift, in turn, further results in higher output errors, creating a vicious cycle in which the model is repeatedly forced to predict on data that increasingly resemble out-of-distribution (OOD) samples.

We substantiate this hypothesis through extensive theoretical and empirical analysis. First, we derive an error propagation formula showing that error growth is determined by the interaction between errors newly generated at the current step and those amplified from previous imperfect predictions. This formulation reveals that even small initial errors inevitably expand into large-scale deviations, akin to the butterfly effect in traditional atmospheric science. Empirically, we find that, the growth of prediction errors is accompanied by a drift in input distributions within the autoregressive pipeline. Much more surprisingly, however, this drift, albeit  visually subtle, is detectable already at the very first step. Together, these findings indicate that our predictions are affected by the vicious feedback cycle from the outset, causing errors to compound throughout the sequence.

Motivated by these insights, we propose Self-Output Fine-Tuning (SOFT), a simple yet effective, plug-and-play training strategy. Our key idea is to directly address the OOD issue by exposing the model to OOD scenarios during training, thereby improving distributional robustness. In particular, SOFT fine-tunes the model on its own one-step predictions, calibrating the biased input distribution and familiarizing the model with inevitable prediction errors, effectively converting OOD scenarios into pseudo in-distribution data. As a model-agnostic tactic, SOFT integrates seamlessly into existing DLWP pipelines, mitigating error growth in long-term forecasting from the very first step. Extensive experiments demonstrate that SOFT achieves state-of-the-art performance on long-horizon forecasting tasks and substantially reduces both prediction errors and distributional discrepancy.
Our contributions can be summarized as follows:

\begin{itemize}
    \item We identify that error growth in autoregressive DLWP models arises from a vicious loop between output errors and shifts in the input distribution.
    
    \item We analyze and reveal that small errors quickly grow, producing out-of-distribution states almost immediately in the autoregressive pipeline.
    \item We introduce \textbf{SOFT}, a plug-and-play strategy that improves long-horizon forecasting by simply fine-tuning the model on its own one-step predicted distributions.
    \item We conduct experiments that demonstrate the potential of SOFT, with open-source release%
    \footnote{\url{https://github.com/yunye0121/SOFT}} to increase reproducibility and impact. In particular, SOFT significantly reduces both prediction errors and distributional discrepancy across long horizons, achieving state-of-the-art performance. The potential underscores the importance of pipeline-level advances in artificial intelligence for science.
\end{itemize}

\noindent This work has been made possible through close engagement within the joint, atmospheric-science-driven research project led by Dr. Hung-Chi Kuo (Department of Atmospheric Sciences) and Dr. Buo-Fu Chen (Center for Weather and Climate Disaster Research) at National Taiwan University. We are grateful for our collaborator's deep domain expertise and strong embrace of artificial intelligence for science, which has been pivotal in identifying this critical challenge in atmospheric science. They fully recognize the potential impact of our proposed solution and are actively engaged in integrating it directly into the operational forecasting pipeline.

\section{Related Work}

\subsection{Deep Learning Weather Prediction and Long-Horizon Training Strategy}

Deep learning methods address the high computational costs and limited efficiency of traditional numerical weather prediction \cite{bauer2015quiet, brotzge2023challenges}. 
Transformer-based architectures dominate the field, including FourCastNet \cite{pathak2022fourcastnet}, Pangu \cite{bi2023accurate}, Aurora \cite{bodnar2025foundation}, Stormer \cite{nguyen2024scaling}, ClimaX \cite{nguyen2023climax}, WeatherGFM \cite{zhao2024weathergfmlearningweathergeneralist}, Fuxi \cite{chen2023fuxi}, and FengWu \cite{chen2023fengwu}. 
Alternatively, GraphCast \cite{lam2023learning} utilizes Graph Neural Networks (GNNs). 

Despite their short-term accuracy, these models suffer from rapid error growth during long-horizon autoregressive inference.

Previous work identifies several primary mitigation strategies.
MultiStep rollout tuning \cite{lam2023graphcast} minimizes error by training on a sequence of predictions rather than a single step.
However, unrolling autoregressive models significantly increases computational costs.
Replay buffers \cite{chen2023fengwu, bodnar2025foundation} attempt to mitigate this overhead by storing training states to simulate sequential goals,
borrowing concepts from deep reinforcement learning.
Alternatively, temporal hierarchical \cite{bi2023accurate} approaches train on larger time steps to reduce the number of required rollouts. 
Both rollout and replay buffer methods try to address error growth via sequential optimization, leading to extra training costs.
Conversely, while temporal hierarchical methods maintain lower deterministic error, they reduce temporal resolution, limiting their applicability for fine-grained meteorological tasks.

\subsection{Error Growth and Distributional Shift}
Error growth, also known as error accumulation, is a fundamental challenge in autoregressive sequence generation, where small early deviations compound to drive the system into unstable states \cite{ranzato2015sequence}.
This phenomenon, formally described as exposure bias \cite{bengio2015scheduled} or covariate shift \cite{arora-etal-2022-exposure}, arises from the discrepancy between training and inference: models are trained conditioned on ground truth but must infer based on their own potentially imperfect predictions.

Common solutions include imitation learning and curriculum strategies.
Scheduled Sampling (SS) \cite{bengio2015scheduled} stochastically mixes ground truth and model predictions in the sequence generated during training.
Dataset Aggregation (Dagger) \cite{ross2011reduction} offers a more rigorous solution by iteratively collecting data from the model's induced distribution and increasing the dataset content after querying an oracle to label these states.
Although theoretically sound and empirically work, these methods pose some significant computational challenges.
Scheduled sampling requires unrolling sequential graphs and DAgger necessitates iterative retraining loops.
While this works for natural language processing as token calculating is cheap, these methods may not work in DLWP as weather variable map is a high-dimensional complex data.

\section{Analysis of Error Growth}
\label{sec:analysisforerrorgrowth}

In this section, we investigate the underlying mechanisms of error growth in autoregressive deep learning weather prediction.
We first establish the problem notation, and then analyze the phenomenon from both theoretical and empirical perspectives.
Our analysis identifies the feedback loop between input distribution shifts and output errors as the dominant factor causing instability in long-horizon predictions.

\subsection{Problem Setup}
\label{sec:problemsetup}

\subsubsection{Weather Forecasting as Regression}
We define the weather state at time $t$ as a tensor $\mathbf{X}_t \in \mathbb{R}^{H \times W \times V}$,
where $H$ and $W$ denote the spatial resolution (latitude and longitude grid) and $V$ represents the number of meteorological variables (e.g., temperature, wind speed, geopotential).

\subsubsection{Autoregressive Rollout}
For long-horizon forecasting, we aim to predict the sequential weather state up to a total horizon $\Delta t$.
Let $f$ be the forecasting model and $\delta t$ be the time interval between inference steps.
The total number of autoregressive steps is $K = \Delta t / \delta t$.
In an autoregressive pipeline, the forecasting model operates recursively, feeding the prediction at step $k$ back as the input for step $k+1$.
The prediction trajectory is generated as:
\begin{equation}
    \hat{\mathbf{X}}_{t+k\delta t} = f_\theta^{k}(\mathbf{X}_{t}), \quad \text{for } k=1, \dots, K
\end{equation}
with a given initial condition $\mathbf{X}_t$.

\subsubsection{Objective}
Our goal is to develop a framework that minimizes the prediction error over the entire horizon.
Formally, we seek to minimize the expected loss over the rollout steps:
\begin{equation}
    \min_\theta \mathop{\mathbb{E}}_{\mathbf{X}_t} \left[ \sum_{k=1}^{K} \mathcal{L}(f_\theta^{k}(\mathbf{X}_{t}), \mathbf{X}_{t+k\delta t}) \right]
\end{equation}
where $\mathcal{L}$ denotes a loss function (e.g., MSE or MAE).

In this work, we are going to answer following researcher questions:
RQ1: What is the root cause of the error growth autogressive DLWP?
RQ2: How to solve it efficiently without brute-forcefully computing increasing?

\subsection{Error Propagation in Autoregressive DLWP}
\label{sec:error_propagation}
\subsubsection{Theoretical Viewpoint}
\label{sec:err_propagation_theoretical}
We first analyze how prediction errors propagate across autoregressive steps,
leading to error growth that deteriorates long-horizon weather forecasting.
While autoregressive models are theoretically sound,
practical implementation inevitably introduces small errors at each forward pass due to data noise \cite{kendall2017uncertainties}, limited model capacity \cite{Goodfellow-et-al-2016}, or architectural bias \cite{teney2024neural}.
We decompose the error at each step into (1) new errors introduced in the current step and (2) errors propagated from previous steps.
For following analysis, we treat the weather state as a flattened vector $\mathbf{x} = \text{vec}(\mathbf{X}) \in \mathbb{R}^d$, allowing us to utilize standard Jacobian notation.

\begin{proposition}[Error Propagation Approximation]
    \label{prop:error_prop}
    Let the ground truth be denoted by $\mathbf{x}_t = g(\mathbf{x}_{t-1})$ and the states from learned model by $\hat{\mathbf{x}}_t = f_\theta(\hat{\mathbf{x}}_{t-1})$.
    We define the \textbf{newly generated step error} at step $t$ as $\boldsymbol{\epsilon}_t \triangleq f_\theta(\mathbf{x}_{t-1}) - \mathbf{x}_t$,
    and the \textbf{error} at step $t$ as $\mathbf{e}_t \triangleq \hat{\mathbf{x}}_t - \mathbf{x}_t$.
    
    Assume the ground truth $g$ and the model error function $\epsilon(\mathbf{x}) \triangleq f_\theta(\mathbf{x}) - g(\mathbf{x})$ are differentiable with Jacobians $\mathbf{J}_g$ and $\mathbf{J}_\epsilon$, respectively.
    If $\mathbf{e}_{t-1}$ is sufficiently small, $\mathbf{e}_{t}$ obeys the recurrence relation:
    \begin{equation} \label{eq:error_recurrence_t}
        \mathbf{e}_t \approx (\mathbf{J}_g(\mathbf{x}_{t-1}) + \mathbf{J}_\epsilon(\mathbf{x}_{t-1})) \mathbf{e}_{t-1} + \boldsymbol{\epsilon}_{t}
    \end{equation}
\end{proposition}

This approximation indicates that $\mathbf{e}_t$ consists of (1) the new single-step error $\boldsymbol{\epsilon}_t$ and (2) a propagated term.
The propagated term transforms the previous error $\mathbf{e}_{t-1}$ through the combined dynamics of the system and the model error Jacobian.
This highlights that error growth is driven not only by the current prediction quality but by the amplification of past errors.
We provide the full derivation in Appendix \ref{sec:app_for_err_acc}.

This approximation can be extended recursively, explaining why errors compound rapidly as the forecast horizon increases.

\begin{corollary}[Error for $n$-steps]
    \label{cor:ea_nstep}
    Using the recurrence from Proposition \ref{prop:error_prop}, the error at step $n$ is:
    \begin{equation}
        \mathbf{e}_n \approx \sum_{i=1}^{n} \underbrace{\left( \prod_{j=i}^{n-1} (\mathbf{J}_g(\mathbf{x}_{j}) + \mathbf{J}_\epsilon(\mathbf{x}_{j})) \right)}_{\mathbf{\Psi}_{n,i}} \boldsymbol{\epsilon}_{i}
    \end{equation}
    where $\mathbf{\Psi}_{n,i}$ is the product of the sum terms from step $i$ to $n$, and the empty product (when $i=n$) is the identity matrix $\mathbf{I}$. 
\end{corollary}

In chaotic systems such as weather, the propagator norm typically satisfies $\|\mathbf{\Psi}_{n,k}\| > 1$.
This implies that past single-step errors $\boldsymbol{\epsilon}_k$ are amplified,
causing the cumulative error to grow rapidly along the generated trajectory.
This mathematical growth implies a critical physical consequence: even small initial errors become significant after multiple autoregressive steps, echoing the \textbf{butterfly effect} often observed in weather systems.

\subsubsection{Empirical Study}
\label{sec:empirical}

Based on the derivation in Section \ref{sec:err_propagation_theoretical}, we posit that there is an inherent risk of error propagation in autoregressive pipelines.
However, the observed error at each step contains both the propagated error described above and the intrinsic difficulty of forecasting that specific time step.
To isolate the impact of propagated error from the intrinsic difficulty,
we conduct a controlled experiment.

Instead of applying autoregressive pipeline, we adopt a \emph{direct prediction} strategy: we set the inference time interval $\delta t$ equal to the horizon $\Delta t$, making the model predict the long term result in a single inference pass.
This strategy eliminates error propagation while preserving other factors.

\begin{figure}[ht]
    \centering
    \includegraphics[width=0.975\linewidth]{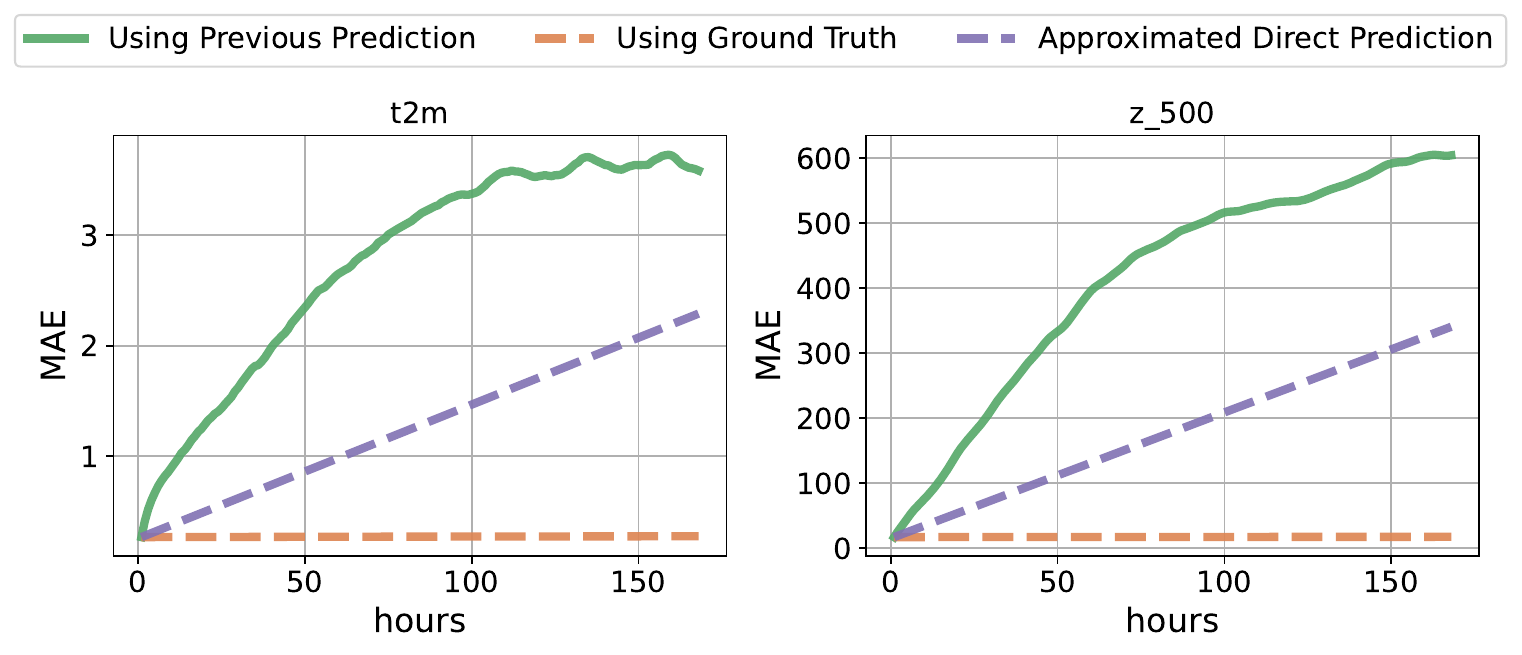}
    \caption{
        \textbf{Impact of error propagation.} 
        The standard autoregressive model (green) accumulates error significantly faster than the direct prediction strategy (purple)
        which isolates the intrinsic forecasting difficulty. 
        The gap between them confirms that error propagation is a primary driver of performance degradation.
    }
    \Description{Two subfigures showing error trend plots for variables t2m and z\_500.}
    \label{fig:emperical_approx}
\end{figure}

Figure~\ref{fig:emperical_approx} presents the results on two representative variables: t2m and z\_500.
Oracle baseline (orange) indicates that using ground truth inputs as the upper bound of performance.
The \emph{Approximated Direct Prediction} (purple) yields a significantly lower error growth rate compared to the standard autoregressive baseline (green).
The performance gap confirms that propagated errors act as a compounding negative factor,
exacerbating the error beyond the intrinsic difficulty of the prediction task itself.

\begin{figure}[htb]
    \centering
    \includegraphics[width = \linewidth]{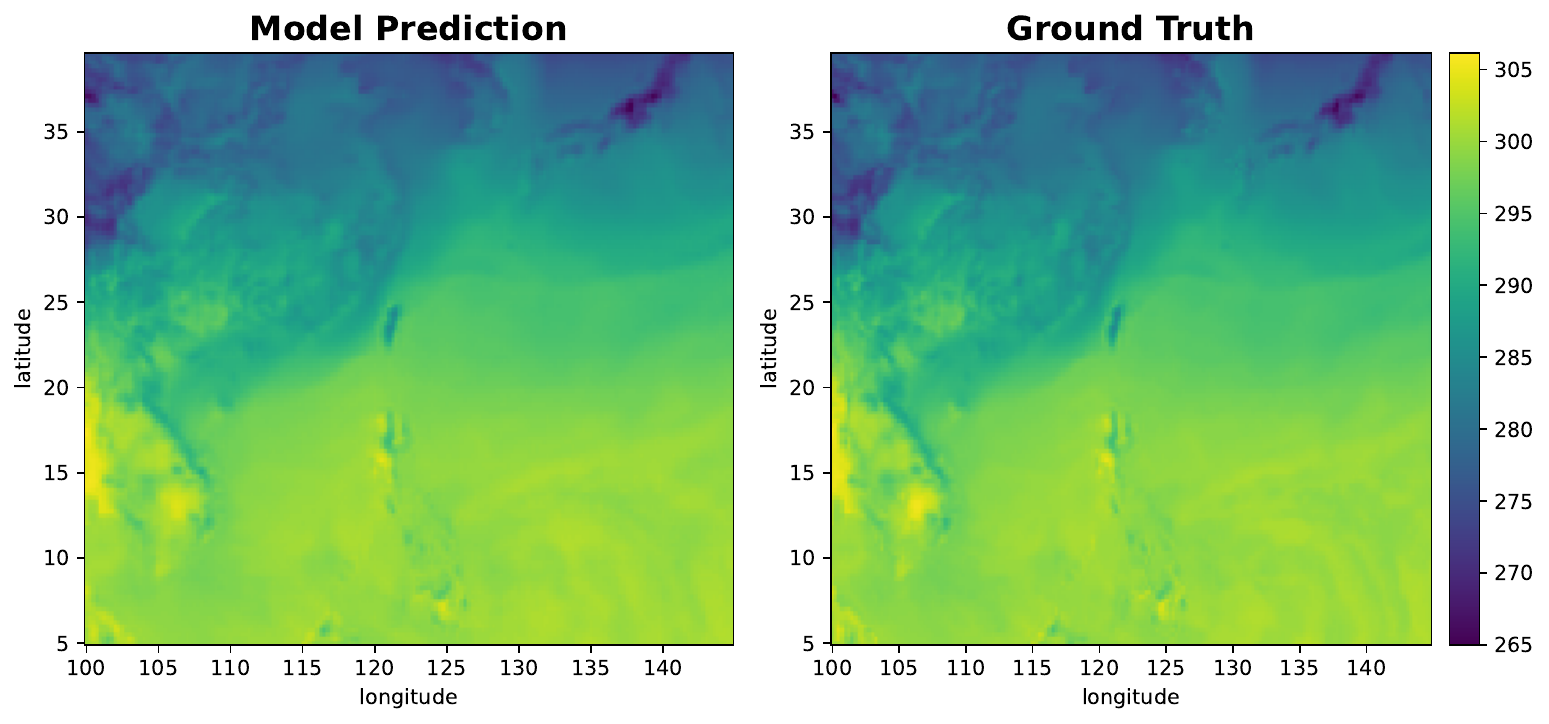}
    \caption{
        Visualization of ground-truth ERA5 (left) and our AuroraTW model prediction (right) at 2023/01/10 08:00. While visually similar to the human eye, the discriminator identifies them as belonging to different distributions.
    }
    \Description{Ground-truth ERA5 and AuroraTW model prediction visualizations at the same timestamp.}
    \label{fig:Truth_and_Generated}
\end{figure}

\subsection{Out-of-Distribution Butterfly effect in DLWP}
\label{sec:OODbegin}
Building on the discussion above, we realized that the autoregressive pipeline brings its own restrictions to weather forecasting tasks.
The key finding in Figure~\ref{fig:emperical_approx} illustrates that using a model's generated prediction as its subsequent input significantly degrades the result.
This causes worse error performance than the inherent difficulty of the task itself, which we empirically showed by having the model directly predict a specific lead time).
This motivates us to study how the model's generated output actually differs from the ground truth and how it affects subsequent predictions.
Figure~\ref{fig:Truth_and_Generated} compares the one-hour single-step prediction alongside the ground truth states at the corresponding time.
Surprisingly, to the expert eye, the predicted fields appear physically consistent and visually indistinguishable from reality.
However, visual fidelity can be misleading, masking high-dimensional discrepancies or artifacts imperceptible to humans \cite{teney2024neural, ulyanov2018deep}.
Deep neural networks, on the other hand, can detect these subtle artifacts inherent to predictions \cite{wang2025exploring}.

\begin{figure}[htb]
    \centering
    \includegraphics[width=0.75\linewidth]{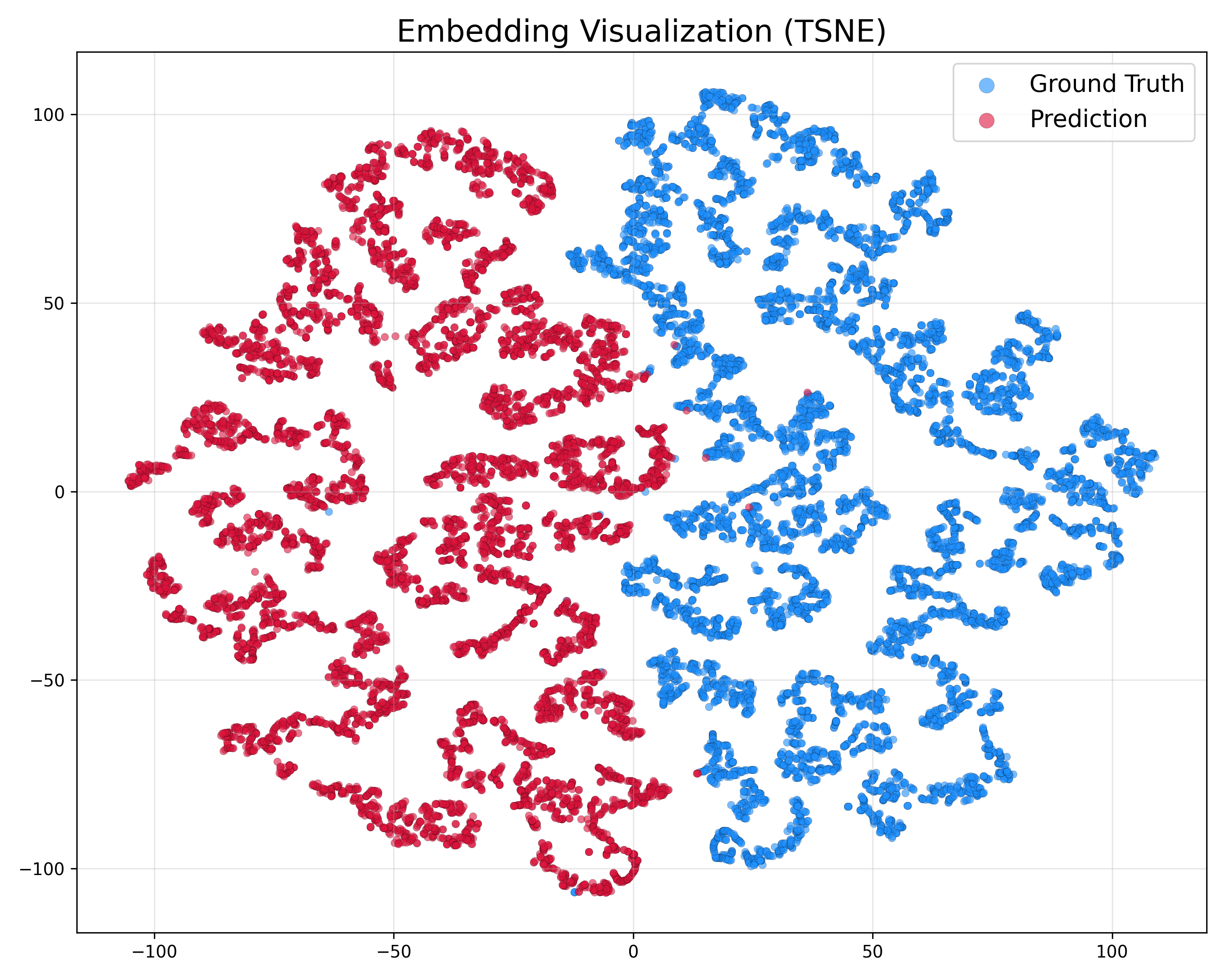}
    \caption{
        t-SNE visualization of embeddings from the AuroraTW encoder. While the output images appear visually similar (see Figure~\ref{fig:Truth_and_Generated}), the model predictions (red) cluster separately from the ground truth (blue). This separation indicates that the predicted state deviates from the data manifold after the first inference step.
    }
    \Description{T-SNE plot showing two distinct clusters for ERA5 data and model predictions.}
    \label{fig:discriminator}
\end{figure}

To quantify these imperceptible differences, we trained a lightweight binary classifier to distinguish between embeddings of original ERA5 data and single-step model predictions (details in Appendix~\ref{sec:probe_details}).
As visualized in Figure~\ref{fig:discriminator}, despite the visual similarity, the classifier achieves near-perfect discrimination ($\approx 100\%$ accuracy) with minimal training.
This reveals that generated predictions do not inhabit the same manifold as the ground truth.
Crucially, this distributional bias is introduced in the very first autoregressive step, setting the stage for the cumulative divergence observed in long-horizon forecasts.
Even though the initial difference is small, the compounding effect amplifies this gap and leads to rapid error divergence at longer horizons.

\subsection{Feedback Loop for Input Distribution Shift and Output Error}
\label{sec:vicious_cycle}

Having quantified the error growth in Section \ref{sec:error_propagation} and \ref{sec:OODbegin}, we now identify the underlying mechanism.
In an autoregressive pipeline, the input for each step depends entirely on the previous output.
This creates a tight coupling between prediction error and the input distribution.
Specifically, errors in the previous step cause the input state for the next step to deviate from the training manifold,
resulting in Out-of-Distribution (OOD) samples.

\begin{figure}[ht]
    \centering
    \includegraphics[width=0.95\linewidth]{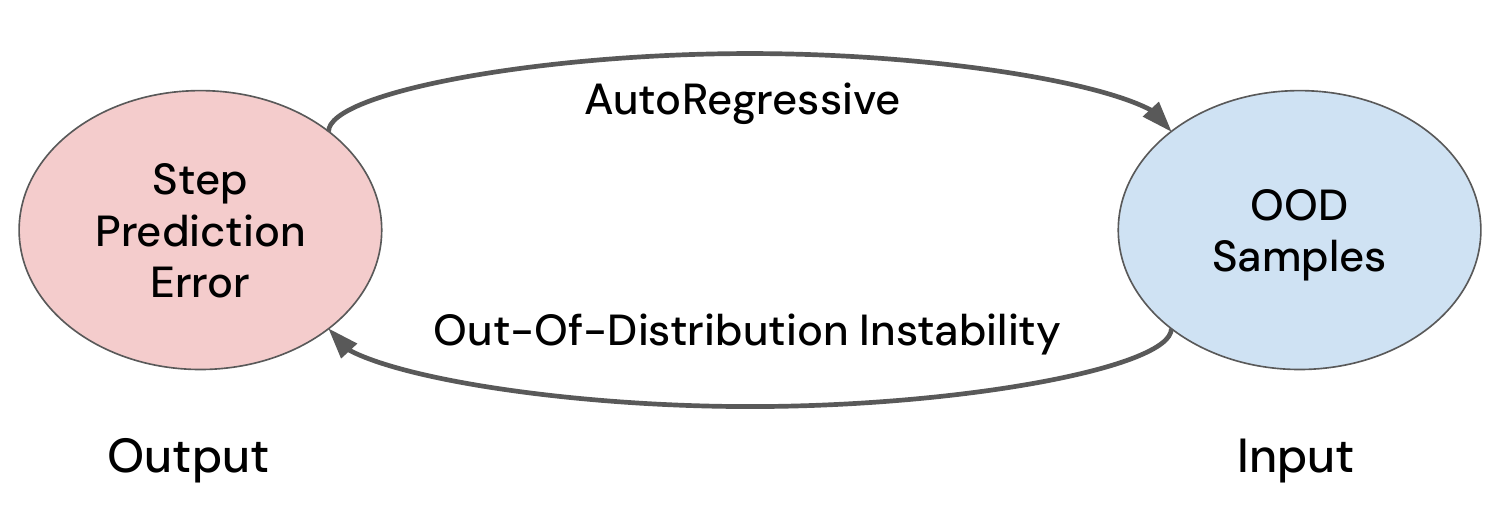}
    \caption{
        \textbf{The Error-Distribution Feedback Loop.} 
        Prediction errors push the subsequent input states out-of-distribution (OOD). 
        This distribution shift causes model instability, leading to larger errors,
        which in turn further deviate the input distribution.
    }
    \Description{Error-Distribution Feedback}
    \label{fig:feedback_loop}
\end{figure}

Since machine learning models typically degrade on OOD inputs, this creates a compounding feedback loop:
(1) prediction errors shift the next input state away from the training distribution;
(2) this distributional shift degrades the model's performance on the next step;
(3) the increased error further biases the input for subsequent steps.

\begin{table}[htb]
    \centering
    \caption{MAE and FID scores between model predictions and ground truth weather states over increasing horizons.}
    \label{tab:FID}
    \begin{tabular}{lrrrrr}
    \toprule
       Variable   & 6hr & 24hr & 96hr & 132hr & 168hr \\
    \midrule
        t2m & 0.599 &  1.057 &  2.141 &  2.404 &  2.609 \\
        u10 & 0.564 &  1.109 &  2.661 &  3.059  & 3.266\\
        t\_850 & 0.397 &  0.770  & 2.128  & 2.499  & 2.789 \\
        z\_500 & 8.535 & 131.595 & 332.714 & 392.902  & 442.777 \\
        \midrule
        FID & 0.007 & 0.085 & 0.552 & 0.868 & 1.105 \\
    \bottomrule
    \end{tabular}
\end{table}

To validate this hypothesis, we quantify the distributional discrepancy between predicted states and the ground truth over increasing autoregressive horizons using the Fréchet Inception Distance (FID).
As shown in Table \ref{tab:FID}, there is a strong positive correlation between autoregressive steps and FID scores,
confirming that the distribution progressively diverges from its learned distribution during the autoregressive process.

\section{Proposed Framework}

\begin{figure*}[t]
    \centering
    \includegraphics[width=0.95\linewidth]{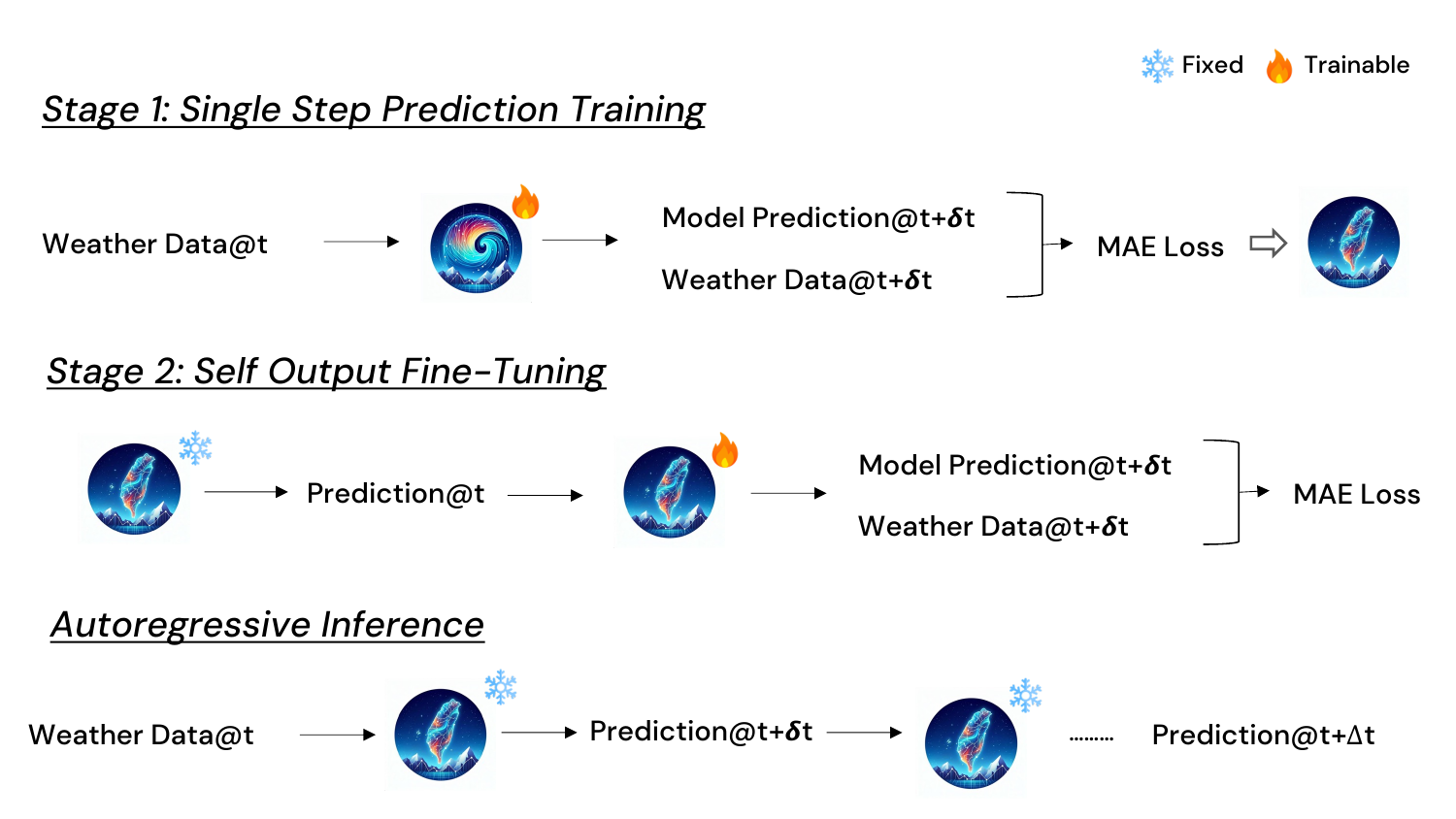}
    \caption{\textbf{Overview of the Self-Output Fine-Tuning (SOFT) framework.} The process consists of two stages: (1) Training a base model on ground truth data to learn single-step dynamics, and (2) Fine-tuning the model on its own generated outputs ($\hat{\mathbf{X}}_t$) to align predictions with the ground truth ($\mathbf{X}_{t+\delta t}$). This alignment mitigates error propagation during autoregressive inference.}
    \label{fig:framework_overview}
    \Description{overview}
\end{figure*}

Based on the analysis in Section \ref{sec:analysisforerrorgrowth}, we attribute error growth in autoregressive pipeline to its error propagation mechanism and negative feedback loops.
Since the model encounters out-of-distribution inputs as early as the first step,
subsequent predictions deviate further from the data manifold.
To address this, we propose \textbf{Self-Output Fine-Tuning (SOFT)}, a simple yet effective finetuning methodology.
Our core idea is to enable forecasting model to adapt to the biased patterns inherent in its own generated outputs.

Unlike computationally expensive approaches that utilize multi-step rollouts or replay buffers,
we construct a single-step distribution alignment task.
Conceptually, this is analogous to \textbf{nipping the error in the bud}.
By correcting small distributional shifts immediately at the single-step level,
we can prevent amplification of errors over long horizons due to the vicious cycle in autoregressive pipeline.
This approach is both strategically targeted and computationally efficient.

\subsection{Stage 1: Training for Single Step Forecasting}

In the first stage of our framework, we train a base model to forecast one step ahead using ERA5 data.
Following previous work \cite{bodnar2025foundation}, we minimize the mean absolute error (\emph{L1Loss}) between predicted states and ground truth grid.
With the notation from Section~\ref{sec:problemsetup}, the training objective is defined as:

\begin{equation}
    \label{eq:st1}
    \theta_{base} = \operatorname*{argmin}_{\theta} \mathop{\mathbb{E}}_{\mathbf{X}_t} \left[ \| f_\theta(\mathbf{X}_{t}) - \mathbf{X}_{t+\delta t} \|_1 \right]
\end{equation}

Training convergence in this stage is critical.
If the base model fails to capture single-step dynamics, the predictions will deviate significantly from ground truth distribution, making the subsequent fine-tuning stage ineffective due to excessive noise.

\subsection{Stage 2: Self-Output Fine-Tuning (SOFT)}

Even if a model achieves low error in Stage 1, it remains susceptible to error growth, as shown in Figure~\ref{fig:overall_figure}.
This occurs because the model is trained on ground truth data but operates on its own generated predictions during autoregressive inference.
As discussed in Section~\ref{sec:OODbegin}, an out-of-distribution signature becomes evident after merely one step.
We therefore propose SOFT, a method that focuses on mitigating the distributional discrepancy between training and inference.
SOFT trains the model on its own predictions,
forcing it to learn how to correct its own artifacts and map biased inputs back to the true future state.
Technically, we transform these OOD samples into pseudo-in-distribution (ID) samples to bridge the distribution gap.

\subsubsection{Generating Biased States.} 

First, we feed a clean sample $\mathbf{X}_{t-\delta t}$ into the pretrained model from Stage 1 (with fixed parameters $\theta_{\text{base}}$) to acquire a predicted input $\hat{\mathbf{X}}_{t}$.
This input contains the model's natural errors and artifacts:
\begin{equation}
    \hat{\mathbf{X}}_{t} = \text{stopgrad}[ f_{\theta_{base}}(\mathbf{X}_{t-\delta t}) ]
\end{equation}
where $\text{stopgrad}$ indicates that model is frozen and gradients are not propagated through this step.

\subsubsection{Fine-Tuning Objective.}
Next, we fine-tune the model to predict the ground truth future $\mathbf{X}_{t+\delta t}$ using the generated input $\hat{\mathbf{X}}_{t}$. The SOFT training objective is defined as:

\begin{equation}
    \label{eq:st2}
    \theta_{SOFT} = \operatorname*{argmin}_{\theta} \mathop{\mathbb{E}}_{\mathbf{X}_t} \left[ \| f_\theta(\hat{\mathbf{X}}_{t}) - \mathbf{X}_{t+\delta t} \|_1 \right]
\end{equation}
where $\hat{\mathbf{X}}_{t} = \text{stopgrad}[f_{\theta_{base}}(\mathbf{X}_{t-\delta t})]$ is the biased input generated by the frozen base model.
Algorithm \ref{alg:soft} (Appendix \ref{sec:appendx_pseudo}) presents the SOFT pseudo-code.

\subsection{Autoregressive Inference}

Once fine-tuned, we deploy the model using a standard autoregressive pipeline.
Recall from Section~\ref{sec:problemsetup} that our goal is to generate a trajectory for a horizon $\Delta t$.
We initialize the process with the ground truth $\mathbf{X}_t$ and recursively apply the learned model:
\begin{equation}
    \hat{\mathbf{X}}_{t+(k)\delta t} = f_{\theta_{SOFT}}(\hat{\mathbf{X}}_{t+(k-1)\delta t}), \quad \text{for } k=1, \dots, K
\end{equation}
with the initialization $\hat{\mathbf{X}}_t = \mathbf{X}_t$.

This loop continues for $K = \Delta t / \delta t$ steps.
Because our model was trained via SOFT, it allows us to extend the rollout horizon significantly (e.g., up to 7 days) without the catastrophic divergence seen in standard baselines.
Furthermore, since we only modify the training input without altering the model architecture,
we do not need to adjust the inference pipeline.

\subsection{Theoretical Insights}

Besides empirical intuition, we also provide a theoretical explanation for why the SOFT loss design works.

Our main goal is to show that, under a linear assumption, the standard sequential rollout training—which optimizes all steps in autoregressive weather prediction—can be upper-bounded by the SOFT loss.

\begin{theorem}
\label{theo:SOFTLOSSUpperBound}
Under a linear setting, for any integer $k \ge 2$, the $k$-step rollout loss $L_{\text{rollout}}$ is upper-bounded by the SOFT loss $L_{\text{SOFT}}$ plus error terms related to model divergence and structural mismatch, such that:
\begin{equation}
    L_{\text{rollout}} \le 2 L_{\text{SOFT}} + 4 \|W\|^2 \|(W - \tilde{W})X_t\|^2 + 4 \|W\|^2 \|(W^{k-1} - W)X_t\|^2
\end{equation}
where $X_t$ denotes the initial condition and $\tilde{W}$ denotes the model weight well-trained by single-step prediction.
\end{theorem}

Through this analysis, we demonstrate that fine-tuning with the model's own output is supported not only by empirical intuition but also by theoretical understanding.
Although this analysis assumes a simplified linear case, it justifies our method from a theoretical perspective and lays the groundwork for future work. 
The detailed derivation is provided in Section~\ref{sec:app_for_SOFTUpperbound}.

\section{Experiments}

\begin{table*}[tb]
\centering
\caption{Main results of spatial MAE for our Aurora 168-hour forecasts across TW, EU, and NA regions. The best numbers in each column are shown in \textbf{bold} and the second best number are \underline{underlined}}.
\label{tab:loss_by_var_168h}
\begin{tabular}{l *{12}{r}}
\toprule			
& \multicolumn{4}{c}{TW} & \multicolumn{4}{c}{EU} & \multicolumn{4}{c}{NA} \\
\cmidrule(lr){2-5} \cmidrule(lr){6-9} \cmidrule(lr){10-13}
Method & t2m & u10 & t\_850 & z\_500 & t2m & 10u & t\_850 & z\_500 & t2m & u10 & t\_850 & z\_500 \\
\midrule

SingleStepPrediction & 3.029 & 3.718 & 3.255 & 504.107 & 5.250 & 3.863 & 5.531 & 1513.869 & 5.296 & 4.247 & 5.916 & 1417.180 \\
+Rollout & 2.898 & 3.447 & 3.050 & 474.561 & 4.828 & 3.700 & 5.298 & 1376.574 & 5.075 & 4.053 & 5.662 & 1316.752 \\
+ReplayBuffer & 2.728 & 3.333 & \underline{2.923} & \underline{446.555} & \textbf{4.566} & \textbf{3.528} & \textbf{4.996} & \textbf{1288.561} & \underline{4.729} & 3.929 & \underline{5.444} & \textbf{1213.615} \\
+GaussianNoiseAug & \underline{2.691} & \underline{3.293} & 2.934 & 459.226 & 5.068 & 3.702 & 5.837 & 1385.477 & 4.885 & \underline{3.877} & 5.752 & 1312.992 \\
+FeatureMatching & 3.034 & 3.844 & 3.249 & 498.506 & 5.317 & 3.967 & 5.605 & 1525.991 & 5.385 & 4.231 & 5.855 & 1354.967 \\
+SOFT (ours) & \textbf{2.609} & \textbf{3.266} & \textbf{2.789} & \textbf{442.777} & \underline{4.762} & \underline{3.570} & \underline{5.187} & \underline{1306.609} & \textbf{4.657} & \textbf{3.824} & \textbf{5.391} & \underline{1217.320} \\


\bottomrule
\end{tabular}
\end{table*}

\label{sec:experiments}

In this section, we evaluate the effectiveness of the proposed SOFT framework on regional weather forecasting.
We structure our empirical analysis around following core questions:

\begin{itemize}
    \item \textbf{Q1:} Does SOFT mitigate error growth in autoregressive inference more effectively than state-of-the-art strategies?
    \item \textbf{Q2:} Does SOFT align the predicted distribution with the ground truth manifold and reduce out-of-distribution shift over long horizons?
    \item \textbf{Q3:} Is SOFT compatible with existing training strategies and diverse model architectures?
\end{itemize}

Section~\ref{sec:expsetup} details the experimental setup.
We address Research Question 1 in Section~\ref{sec:main_results}, Question 2 in Section~\ref{distribution}, and Questions 3 in Section~\ref{sec:ablationstudy}.

\subsection{Experimental Setup}
\label{sec:expsetup}

\subsubsection{Datasets} We evaluate our framework using the ERA5 reanalysis dataset \cite{hersbach2020era5}, specifically curating a regional subset centered on East Asia surrounding Taiwan, North America and Europe ($0.25^\circ$ resolution, detailed in Table \ref{tab:era5_details}, Appendix \ref{sec:expsetup_appendix_data}) to simulate a practical local deployment.

\subsubsection{Architectures} We utilize the pre-trained Microsoft Aurora \cite{bodnar2025foundation} as our primary backbone, adapting it to our region via standard supervised learning. To demonstrate architectural generalization, we also employ a re-implementation of Pangu-Weather \cite{bi2023accurate}, trained from scratch on our regional split.

\subsubsection{Metrics} We consider both deterministic error value and distributional performance in this paper. We assess deterministic forecast accuracy using spatially-averaged Mean Absolute Error (MAE). To quantify distributional realism and the butterfly effect drift, we employ the Fréchet Inception Distance (FID), measuring the Wasserstein-2 distance between the feature embeddings of predicted and ground-truth atmospheric states.

\subsubsection{Baselines}
\label{exp:baseline}
To validate SOFT effectiveness, we compare it against common paradigms in long term training strategy in DLWP.

\begin{itemize}
    \item \textbf{Single Step Prediction Training:} The baseline model only trained with standard next-step supervision ($t \rightarrow t+ \delta t$). This serves as the lower bound, representing a model with no specific long-term optimization.
    \item \textbf{Rollout Tuning:} The dominant strategy in SOTA models~\cite{lam2023graphcast}, where predictions are unrolled for $K$ steps during training. While effective, it incurs linear growth for both memory and training cost.
    \item \textbf{Replay Buffer:} Maintain a buffer of intermediate states to approximate long-horizon training without direct unrolling.
    \item \textbf{Feature Matching Loss:} A Generative Adversarial Network framework where a discriminator penalized the model embedding for deviations from the realistic atmospheric manifold.
    \item \textbf{Gaussian Noise Augmentation:} We inject Gaussian noise into inputs during training to prevent overfitting to deterministic trajectories and improve robustness against distributional shifts.
\end{itemize}

Appendix \ref{sec:expsetup_appendix} lists the implementation details and hyperparameters for all baselines.

\begin{table}[htb]
\caption{Win rate table for different methods comparison.}
\centering
\begin{tabular}{lrrrrr}
\toprule
 Method & 6hr & 24hr & 96hr & 132hr & 168hr \\
\midrule
SingleStepPrediction & 0 & 0 & 0 & 0 & 0 \\
+Rollout & 2 & 1 & 0 & 0 & 0 \\
+ReplayBuffer & 5 & 11 & 1 & 1 & 1 \\
+GaussianNoiseAug & 0 & 1 & 6 & 7 & 9 \\
+FeatureMatching & 0 & 0 & 0 & 0 & 0 \\
+SOFT (ours) & \textbf{37} & \textbf{31} & \textbf{37} & \textbf{36} & \textbf{34} \\
\bottomrule
\end{tabular}
\label{tab:win_rate_comparsion}
\end{table}

\subsection{Comparison with State-of-the-Arts}
\label{sec:main_results}

\subsubsection{Quantitative MAE Results}

Table~\ref{tab:loss_by_var_168h} presents the MAE performance for 168-hour autoregressive forecasting.
We report results on four representative variables: t2m, u10, t\_850, and z\_500.
These are standard metrics for long-range prediction and are critical for human activity.

The proposed SOFT framework achieves competitive performance across all lead times and variables.
While MultiStep rollout and ReplayBuffer strategies are effective,
they incur additional computational costs and optimization instability.
Conversely, distributional enhancement methods such as Feature Matching and Gaussian Noise Augmentation struggle to fully correct distribution mismatch.
SOFT successfully balances computational efficiency with forecasting accuracy.

\subsubsection{Comparison across all variables}
We also report the number of variables that rank first in the 168-hour forecast.
As shown in Table~\ref{tab:win_rate_comparsion}, the SOFT series wins across most variables.
This shows that our method is not limited to specific variables or leading times, but provides a general improvement for these models.

\begin{figure}[htb]
    \centering
    \includegraphics[width=.85\linewidth]{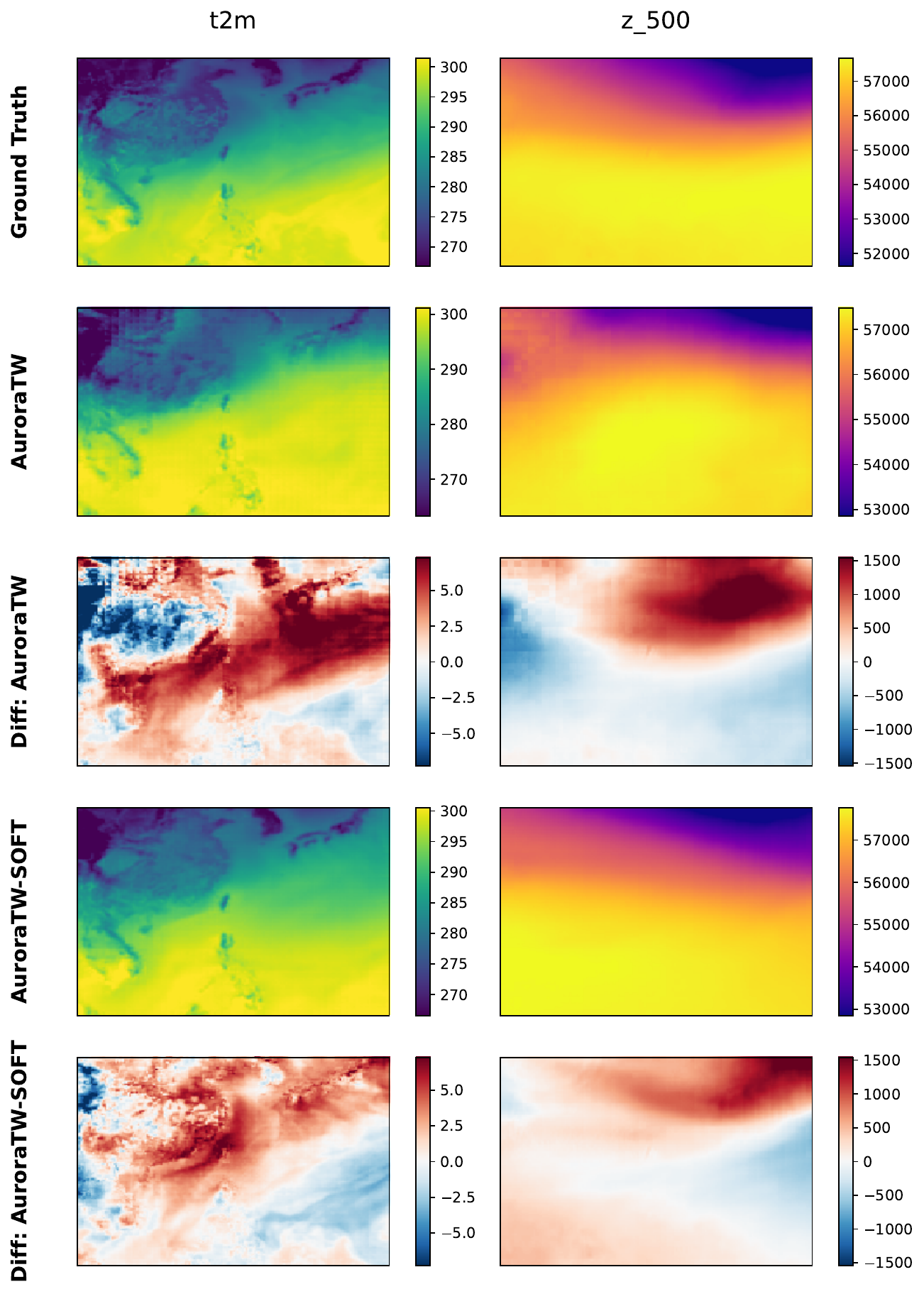}
    \caption{\textbf{Prediction Error Visualization.} Visualization of absolute error for AuroraTW and AuroraTW-SOFT relative to ERA5 (ground truth) at a 168-hour lead time. The proposed SOFT method (rows 4-5) shows reduced error compared to the baseline (rows 2-3).}
    \Description{visualize}
    \label{fig:error_visualization}
\end{figure}

\subsubsection{Visualization}
Figure~\ref{fig:error_visualization} shows the error distribution for a representative case study on 2023/01/28 12:00 UTC at a 168-hour lead time. The baseline model exhibits large errors (darker regions), whereas SOFT shows significantly reduced error.

\subsection{Mitigating Distributional Gap}
\label{distribution}

\begin{table}[ht!]
    \centering
    \caption{\textbf{FID evaluation for three regions at a 168-hour horizon.} Lower values indicate better alignment with the realistic atmospheric manifold.}
    \begin{tabular}{lccc}
    \toprule
    Method & TW & EU & NA \\
    \midrule
        SingleStepPrediction & 1.105 & 2.741 & 3.431 \\
        +Rollout     & 0.838 & 1.509 & 2.510 \\
        +ReplayBuffer  & 0.790 & 1.159 & 1.915 \\
        +GaussianNoise & 0.886 & 2.456 & 2.064 \\
        +FeatureMatching & 1.173 & 2.663 & 2.922 \\
        +SOFT (ours)   & 0.737 & 1.233 & 1.556 \\
    \bottomrule
    \end{tabular}
    \label{tab:FID_region}
\end{table}

A core premise of this work is that autoregressive models suffer from distributional shift — where predictions drift away from the manifold of natural weather states over time.
We therefore aim to demonstrate that SOFT effectively address the distribution problem caused by the autoregressive compounding errors.

We evaluate the Fréchet Inception Distance (FID) between predicted rollout sequences and ground truth sequences.
A lower FID indicates that the predicted weather patterns are distributionally closer to the ground truth.
As shown in Table~\ref{tab:FID_region}, both SOFT and the other strategies achieve significantly lower FID scores at later forecast stages compared to baselines.
This supports our hypothesis that error growth is driven by distributional drift, and SOFT effectively mitigates this issue.

\subsection{Ablation Study}
\label{sec:ablationstudy}

\subsubsection{Compatibility with Different Strategies}
While SOFT is effective as a standalone method, we investigate whether it complements other long-horizon strategies.
We apply SOFT on top of checkpoints pretrained with the baseline methods discussed in Section~\ref{exp:baseline}.

As shown in Table~\ref{tab:loss_add}, SOFT improves performance in the majority of tested scenarios, effectively enhancing strategies like multi-step rollout and replay buffer with minimal additional training.
This suggests that SOFT functions as a robust plug-and-play extension that enhances distributional robustness for various underlying training paradigms.

\begin{table}[htb]
\centering
\caption{The study for SOFT generalization ability. Values in italics indicate the relative percentage change compared to the respective baseline for 168 hour forecast.}
\label{tab:loss_add}
\begin{tabular}{l *{4}{r}}
\toprule			
Method & t2m & u10 & t\_850 & z\_500 \\
\midrule
SingleStepPrediction & 3.029 & 3.718 & 3.255 & 504.107 \\
+SOFT & 2.609 & 3.266 & 2.789 & 442.777 \\
\quad \textit{(\% Change)} & \textit{(-13.9\%)} & \textit{(-12.2\%)} & \textit{(-14.3\%)} & \textit{(-12.2\%)} \\
\midrule
+Rollout & 2.898 & 3.447 & 3.050 & 474.561 \\
+Rollout+SOFT &  2.785 & 3.362 & 2.901 & 468.248 \\
\quad \textit{(\% Change)} & \textit{(-3.9\%)} & \textit{(-2.5\%)} & \textit{(-4.9\%)} & \textit{(-1.3\%)} \\
\midrule
+ReplayBuffer & 2.728 & 3.333 & 2.923 & 446.555 \\
+ReplayBuffer+SOFT & 2.694 & 3.311 & 2.876 & 459.404 \\
\quad \textit{(\% Change)} & \textit{(-1.2\%)} & \textit{(-0.7\%)} & \textit{(-1.6\%)} & \textit{(+2.9\%)} \\
\bottomrule
\end{tabular}
\end{table}

\subsubsection{Compatibility with Different Architectures}
We apply SOFT to the Pangu-Weather architecture to demonstrate architectural generalization.
Table~\ref{tab:ablation_arch_table} shows that Pangu-Weather achieves consistent gains when SOFT is applied,
confirming the method's versatility across different backbone architectures.

\begin{table}[htb]
\centering
\caption{Ablation study for applying SOFT on different architecture. Values in italics indicate the relative percentage change compared to the PanguTW baseline for 168 hour forecasts.}
\label{tab:ablation_arch_table}
\setlength{\tabcolsep}{4pt}
\begin{tabular}{l r r r r}
\toprule
& t2m & 10u & t\_850 & z\_500 \\
\midrule
PanguTW & 29.389 & 40.051 & 10.749 & 2333.240 \\
+SOFT (ours) & \textbf{27.902} & \textbf{25.607} & \textbf{4.451} & \textbf{1475.974} \\
\quad \textit{(\% Change)} & \textit{(-5.1\%)} & \textit{(-36.1\%)} & \textit{(-58.6\%)} & \textit{(-36.7\%)} \\
\bottomrule
\end{tabular}
\end{table}

\section{Conclusion}
\label{sec:conclusion}

In this paper, we identify that error accumulation in autoregressive deep learning weather prediction is fundamentally driven by distributional shift.
We theoretically analyze this phenomenon, showing that propagating errors progressively accumulate, pushing subsequent predictions away from the ground truth manifold.
Empirical evidence further demonstrates that the distribution begins to deviate from the ground truth trajectory as early as the first prediction step.

To mitigate this distributional mismatch between training and inference, we propose Self-Output Fine-Tuning (SOFT). This simple yet effective strategy explicitly exposes the model to its own one-step predictive bias during training, effectively bridging the distribution gap.
Extensive experiments on the ERA5 dataset validate that SOFT performs competitively across various weather variables and lead times. By aligning step-wise distributions, we successfully reduce both deterministic and distributional errors, as evidenced by improved MAE and FID scores.

Crucially, SOFT is architecture-agnostic and computationally efficient, prioritizing a step-wise training objective over the sequential objectives commonly used in prior work.
This design is practical in terms of cost and ensures ground-target adherence, consistent with our theoretical analysis. We demonstrate that this step-wise goal remains competitive with sequential objectives while utilizing the same computational budget.
These findings suggest that long-range forecasting research should prioritize distributional alignment over complex architectural modifications.
We hope this work encourages the community to view error accumulation through the lens of data distribution, paving the way for more reliable weather models in climate science.


\section*{Limitations and Ethical Considerations}
While validated on standard ERA5 benchmarks, our experiments are currently limited to a regional East Asia subset. Global-scale deployment remains future work. Ethically, we anticipate positive societal impacts on disaster preparedness and climate monitoring. Furthermore, our framework contributes to Green AI by significantly reducing training costs compared to multi-step rollouts. Utilizing public reanalysis data without human subjects, this work poses no privacy, discrimination, or conflict of interest risks.

\section*{Generative AI Disclosure}
We utilized generative AI tools to assist with the preparation of this manuscript. Specifically, these tools were used to refine the writing (grammar and clarity), generate components of the framework overview diagram, and assist with coding tasks. The authors reviewed all AI-generated content to ensure accuracy and retain full responsibility for the paper's content.

\bibliographystyle{ACM-Reference-Format}
\bibliography{ref}



\appendix
\section*{Appendix}

\section{Derivation of Error Propagation} \label{sec:app_for_err_acc}

In this section, we provide the detailed derivation for Proposition \ref{prop:error_prop}.

\paragraph{Definitions.}
Recall the definitions from the main text:
\begin{itemize}
    \item Ground truth dynamics: $\mathbf{x}_t = g(\mathbf{x}_{t-1})$.
    \item Learned model: $\hat{\mathbf{x}}_t = f_\theta(\hat{\mathbf{x}}_{t-1})$.
    \item Model error function: $\epsilon(\mathbf{x}) \triangleq f_\theta(\mathbf{x}) - g(\mathbf{x})$.
    \item Cumulative error at step  $t$: $\mathbf{e}_t \triangleq \hat{\mathbf{x}}_t - \mathbf{x}_t$.
    \item Single-step error at step $t$: $\boldsymbol{\epsilon}_t \triangleq f_\theta(\mathbf{x}_{t-1}) - \mathbf{x}_t$.
\end{itemize}

\begin{proof}
We start by expanding the cumulative error at step $t$:
\begin{equation}
    \mathbf{e}_t = \hat{\mathbf{x}}_t - \mathbf{x}_t
\end{equation}
Substituting the model dynamics for $\hat{\mathbf{x}}_t$ and the ground truth dynamics for $\mathbf{x}_t$:
\begin{equation}
    \mathbf{e}_t = f_\theta(\hat{\mathbf{x}}_{t-1}) - g(\mathbf{x}_{t-1})
\end{equation}
Since $\hat{\mathbf{x}}_{t-1} = \mathbf{x}_{t-1} + \mathbf{e}_{t-1}$, we can rewrite the first term as $f_\theta(\mathbf{x}_{t-1} + \mathbf{e}_{t-1})$. Assuming $\mathbf{e}_{t-1}$ is small, we perform a first-order Taylor expansion of $f_\theta$ around the ground truth state $\mathbf{x}_{t-1}$:
\begin{equation}
    f_\theta(\mathbf{x}_{t-1} + \mathbf{e}_{t-1}) \approx f_\theta(\mathbf{x}_{t-1}) + \mathbf{J}_{f}(\mathbf{x}_{t-1}) \mathbf{e}_{t-1}
\end{equation}
where $\mathbf{J}_f$ is the Jacobian of the model $f_\theta$. Substituting this back into the error equation:
\begin{equation}
    \mathbf{e}_t \approx f_\theta(\mathbf{x}_{t-1}) + \mathbf{J}_{f}(\mathbf{x}_{t-1}) \mathbf{e}_{t-1} - g(\mathbf{x}_{t-1})
\end{equation}
We can rearrange the terms to group the single-step error:
\begin{equation}
    \mathbf{e}_t \approx \underbrace{(f_\theta(\mathbf{x}_{t-1}) - g(\mathbf{x}_{t-1}))}_{\boldsymbol{\epsilon}_t} + \mathbf{J}_{f}(\mathbf{x}_{t-1}) \mathbf{e}_{t-1}
\end{equation}
Recall that by definition, the model output is the sum of the ground truth and the error function: $f_\theta(\mathbf{x}) = g(\mathbf{x}) + \epsilon(\mathbf{x})$. 
By the linearity of differentiation, the Jacobian of the model $\mathbf{J}_f$ decomposes into:
\begin{equation} \label{eq:jacobian_decomp}
    \mathbf{J}_f(\mathbf{x}) = \nabla_{\mathbf{x}} (g(\mathbf{x}) + \epsilon(\mathbf{x})) = \mathbf{J}_g(\mathbf{x}) + \mathbf{J}_\epsilon(\mathbf{x}).
\end{equation}
We previously derived the error update rule using the model Jacobian: 
$\mathbf{e}_t \approx \boldsymbol{\epsilon}_t + \mathbf{J}_f(\mathbf{x}_{t-1}) \mathbf{e}_{t-1}$.
Substituting Eq. \ref{eq:jacobian_decomp} into this update rule yields the final recurrence relation:
\begin{equation}
    \mathbf{e}_t \approx \boldsymbol{\epsilon}_t + (\mathbf{J}_g(\mathbf{x}_{t-1}) + \mathbf{J}_\epsilon(\mathbf{x}_{t-1})) \mathbf{e}_{t-1}.
\end{equation}

\end{proof}

\section{Derivation of SOFT Loss Upper Bound} \label{sec:app_for_SOFTUpperbound}

In this section, we provide the detail process for getting theorem \ref{theo:SOFTLOSSUpperBound}.

\begin{proof}
We begin by expressing the $k$-step rollout loss and applying the Triangle Inequality by adding and subtracting the prediction of the surrogate model, $W \tilde{W} X_t$:
\begin{equation}
    \|W^k X_t - X_{t+2}\| \le \|W^k X_t - W \tilde{W} X_t\| + \|W \tilde{W} X_t - X_{t+2}\|
\end{equation}

We recognize the second term on the right-hand side as the square root of our SOFT surrogate loss, $\sqrt{L_{\text{SOFT}}}$. For the first term, we factor out the matrix $W$:
\begin{equation}
    \|W^k X_t - W \tilde{W} X_t\| = \|W(W^{k-1} - \tilde{W})X_t\|
\end{equation}

Using the submultiplicativity property of the chosen matrix norms, we can bound this term:
\begin{equation}
    \|W(W^{k-1} - \tilde{W})X_t\| \le \|W\| \cdot \|(W^{k-1} - \tilde{W})X_t\|
\end{equation}

To isolate the divergence between the current model $W$ and the fixed model $\tilde{W}$, we add and subtract $W$ within the rightmost term:
\begin{equation}
    W^{k-1} - \tilde{W} = (W^{k-1} - W) + (W - \tilde{W})
\end{equation}

Applying the Triangle Inequality again yields:
\begin{equation}
    \|(W^{k-1} - \tilde{W})X_t\| \le \|(W^{k-1} - W)X_t\| + \|(W - \tilde{W})X_t\|
\end{equation}

Substituting this back into our original inequality, we have:
\begin{equation}
    \|W^k X_t - X_{t+2}\| \le \sqrt{L_{\text{SOFT}}} + \|W\| \left( \|(W^{k-1} - W)X_t\| + \|(W - \tilde{W})X_t\| \right)
\end{equation}

To relate this back to the squared loss functions, we square both sides. We utilize the standard algebraic inequality $(a+b)^2 \le 2a^2 + 2b^2$ twice. 

First, applying it to split the surrogate loss from the error terms:
\begin{equation}
    L_{\text{rollout}} \le 2 L_{\text{SOFT}} + 2 \|W\|^2 \left( \|(W^{k-1} - W)X_t\| + \|(W - \tilde{W})X_t\| \right)^2
\end{equation}

Second, applying it to expand the two inner error terms:
\begin{equation}
    \left( \|(W^{k-1} - W)X_t\| + \|(W - \tilde{W})X_t\| \right)^2 \le 2 \|(W^{k-1} - W)X_t\|^2 + 2 \|(W - \tilde{W})X_t\|^2
\end{equation}

Substituting this expansion back into the squared bound yields the final result:
\begin{equation}
    L_{\text{rollout}} \le 2 L_{\text{SOFT}} + 4 \|W\|^2 \|(W - \tilde{W})X_t\|^2 + 4 \|W\|^2 \|(W^{k-1} - W)X_t\|^2
\end{equation}
This completes the proof.
\end{proof}

\section{Experimental Setup}
\label{sec:expsetup_appendix}

\subsection{Datasets}
\label{sec:expsetup_appendix_data}
We use the ERA5 reanalysis dataset, a standard benchmark for data-driven weather forecasting.
ERA5 provides hourly estimates of atmospheric variables on a $0.25^\circ$ grid.
We retrieved the data from the ECMWF Climate Data Store (CDS) using the \texttt{cdsapi} client, downloading all files in NetCDF (\texttt{.nc}) format.

We curate regional subsets of ERA5 for different areas, including East Asia surrounding Taiwan, North America and Europe. Details are provided in Table~\ref{tab:era5_details}.
This setup reflects a practical deployment scenario where global foundation models are adapted to specific local regions.

\begin{table}[htb]
\caption{ERA5 Dataset Configuration for Regional Fine-tuning.}
\centering
\begin{tabular}{ll}
\toprule
\textbf{Attribute}           & \textbf{Specification}                       \\
\midrule
Spatial Resolution           & $0.25^\circ$ ($\approx 28$ km)               \\
Temporal Resolution          & 1 hour                                       \\
TW Region                & $5^\circ$N--$39.75^\circ$N, $100^\circ$E--$144.75^\circ$E \\
EU Region                & $35^\circ$N--$69.75^\circ$N, $0^\circ$E--$44.75^\circ$E \\
NA Region                & $30^\circ$N--$64.75^\circ$N, $115^\circ$E--$159.75^\circ$E \\
Grid Size                    & $140 \times 180$                             \\
\midrule
Training Period              & 2013--2018                                   \\
Validation Period            & 2022                                         \\
Testing Period               & 2023                                         \\
\midrule
Upper-air Variables          & u, v, t, q, z                                \\
Pressure Levels (hPa)        & 50, 250, 500, 600, 700, 850, 925             \\
Surface Variables            & u{10}, v{10}, t{2m}, \text{msl}              \\
\bottomrule
\end{tabular}
\label{tab:era5_details}
\end{table}

\subsection{Model Architecture}

We utilize the pre-trained Microsoft Aurora model as our backbone.
Aurora employs a 3D Swin Transformer within a Perceiver IO framework to capture multi-scale atmospheric dynamics.
We first adapt this global model to our regional domain via standard one-step supervised learning.
This provides a strong baseline for subsequent fine-tuning strategies.

To evaluate architectural generalization (RQ3), we also employ Pangu-Weather.
Pangu utilizes a 3D Earth-Specific Transformer (3DEST) to process volumetric data.
As the official implementation is not open-source,
we re-implemented the architecture based on the published pseudo-code and trained it from scratch on our regional dataset.

\subsection{Evaluation Metrics}
Our primary metric for forecast accuracy is the Mean Absolute Error (MAE),
averaged over spatial dimensions.
Let $\mathbf{X}_t \in \mathbb{R}^{H \times W \times C}$ be the ground-truth state at time $t$, and $\hat{\mathbf{X}}_t$ be the prediction.
The loss is defined as:
\[
\mathcal{L}_{\text{MAE}} = \frac{1}{H \cdot W \cdot C} \sum_{h=1}^{H} \sum_{w=1}^{W} \sum_{c=1}^{C} \left| \mathbf{X}_t^{(h,w,c)} - \hat{\mathbf{X}}_t^{(h,w,c)} \right|
\]
where $H, W, C$ denote latitude, longitude, and channel dimensions, respectively.

To assess whether predicted weather patterns remain realistic over long horizons, we employ the FID score.
FID measures the Wasserstein-2 distance between the multivariate Gaussian distributions of the ground truth ($\mathcal{D}_r$) and predicted ($\mathcal{D}_p$) data in a deep feature space.
It is computed as:
\[
\text{FID} = \|\mu_r - \mu_p\|^2 + \text{Tr}\left( \Sigma_r + \Sigma_p - 2(\Sigma_r \Sigma_p)^{1/2} \right)
\]
where $(\mu_r, \Sigma_r)$ and $(\mu_p, \Sigma_p)$ denote the mean and covariance of the feature embeddings for the real and predicted samples, respectively.
Lower FID values indicate better alignment with the ground truth distribution, suggesting reduced distributional shift during autoregressive inference.

\subsection{Implementation Details}
\label{sec:implementation}

\subsubsection{Baseline Training Setting}
We first fine-tune the global Aurora foundation model on the ERA5 subset for 400 epochs to obtain the well-trained model which works for single step weather prediction.
We then use this as the base for further fine-tuning with different long-term strategies for an additional 50 epochs.

We follow the default normalization settings from the Aurora repository.
We adopt a learning rate schedule with linear warmup (first 10\% of steps) followed by cosine annealing.
The peak learning rate is $3 \times 10^{-5}$ with a weight decay of $1 \times 10^{-3}$.
Batch size is set to 32 except rollout, which is configured to 16 with gradient accumulation step 2 due to GPU memory consistency. 
All experiments use 8 NVIDIA H200 GPUs.
Single step prediction training requires approximately 24 hours, while each subsequent strategy requires approximately 4 to 6 hours.

\subsubsection{Baseline Hyperparameters}
Key hyperparameters for the baseline comparisons are as follows:
\begin{itemize}
    \item \textbf{Rollout:} Rollout horizon $K=2$ steps.
    \item \textbf{Replay Buffer:} Buffer size 250, fine-tuning rate 1, replay start step 0.
    \item \textbf{Feature Matching Loss:} ResNet50 discriminator (via \texttt{timm}). The discriminator is updated every 10 epochs for a duration of 10 epochs. We adopt all layer discriminator embeddings.
    \item \textbf{Gaussian Noise Augmentation:} Standard deviation $\sigma = 0.05$.
\end{itemize}

\subsubsection{Evaluation Metrics}
To calculate the Fréchet Inception Distance (FID),
we use feature embeddings extracted from the 3D Swin Transformer backbone encoder of \textit{AuroraTW}.
We use this encoder for all FID calculations reported in this work.

\subsubsection{Discriminator Probe Training}
\label{sec:probe_details}
To quantify distribution shift, we trained a binary classifier on embeddings extracted from the frozen AuroraTW encoder.
The probe is a 2 layer residual MLP with GELU, 1D Batch Normalization, and Dropout.
It will be trained for 50 epochs using Adam and binary cross-entropy loss.
On a balanced test set of ERA5 and single-step predictions
the model achieved near 100\% accuracy, confirming the embeddings are distinct.

\subsubsection{Ablation Study Training Setting}
For the ablation studies to test SOFT generalization, we apply the SOFT strategy for an additional 25 epochs based on current existing checkpoint.
For Pangu-Weather, the single-step training mirrors the Aurora finetuning protocol but is trained from scratch without a global checkpoint.
Other training factor (e.g. learning rate, optimizer and weight decay setting) remain unchanged.

\begin{algorithm}[htb]
\caption{Pseudo code of Self-Output Fine-Tuning (SOFT)}
\label{alg:soft}
\begin{algorithmic}[1]
\Require Pretrained base model $\theta_{base}$
\State Initialize trainable model $f_{\theta}$ with weights $\theta \leftarrow \theta_{base}$
\State Initialize frozen generator $f_{\text{gen}}$ with weights $\theta_{base}$
\While{not converged}
    \State Sample batch $(\mathbf{X}_{t-\delta t}, \mathbf{X}_t, \mathbf{X}_{t+\delta t})$
    \State \textcolor{gray}{// 1. Generate gradient-detached states (using frozen model)}
    \State $\hat{\mathbf{X}}_t \leftarrow \text{stopgrad}[f_{\text{gen}}(\mathbf{X}_{t-\delta t})]$
    \State \textcolor{gray}{// 2. Predict future using biased input and compute loss}
    \State $\mathcal{L} \leftarrow \| f_{\theta}(\hat{\mathbf{X}}_t) - \mathbf{X}_{t+\delta t} \|_1$
    \State \textcolor{gray}{// 3. Update parameters}
    \State $\theta \leftarrow \text{Optimizer}(\theta, \nabla_{\theta} \mathcal{L})$
\EndWhile
\State \Return $\theta$
\end{algorithmic}
\end{algorithm}

\subsection{Pseudo Code}
\label{sec:appendx_pseudo}

We provide the detailed pseudo-code in Algorithm~\ref{alg:soft} to further clarify the implementation of Self-Output Fine-Tuning (SOFT).
SOFT uses a frozen copy of the pretrained model ($f_{\text{gen}}$) to generate an intermediate state $\hat{\mathbf{X}}_t$.
By training the model $f_{\theta}$ to predict the future state $\mathbf{X}_{t+\delta t}$ from this generated input,
we improve the model's robustness to its own prediction errors without requiring expensive multi-step backpropagation.

\section{Additional Experimental Results}

\subsection{Visualization}

Figures \ref{fig:vis_app_t2m}, \ref{fig:vis_app_u10}, \ref{fig:vis_app_t_850}, and \ref{fig:vis_app_z_500} provide comparative visualizations of the ground truth, the AuroraTW baseline, and AuroraTW-SOFT predictions. We visualize the same variables analyzed in the main paper (t2m, u10, t\_850, and z\_500) using an initial start time of 2023/06/21 12:00. The predictions are shown for lead times of 6, 24, 96, and 168 hours. Consistent with our primary findings, the SOFT method reduces the error growth observed in the baseline over these extended horizons.

\clearpage

\begin{figure*}[p]
    \centering
    \vspace*{\fill}
    \includegraphics[width=0.95\linewidth]{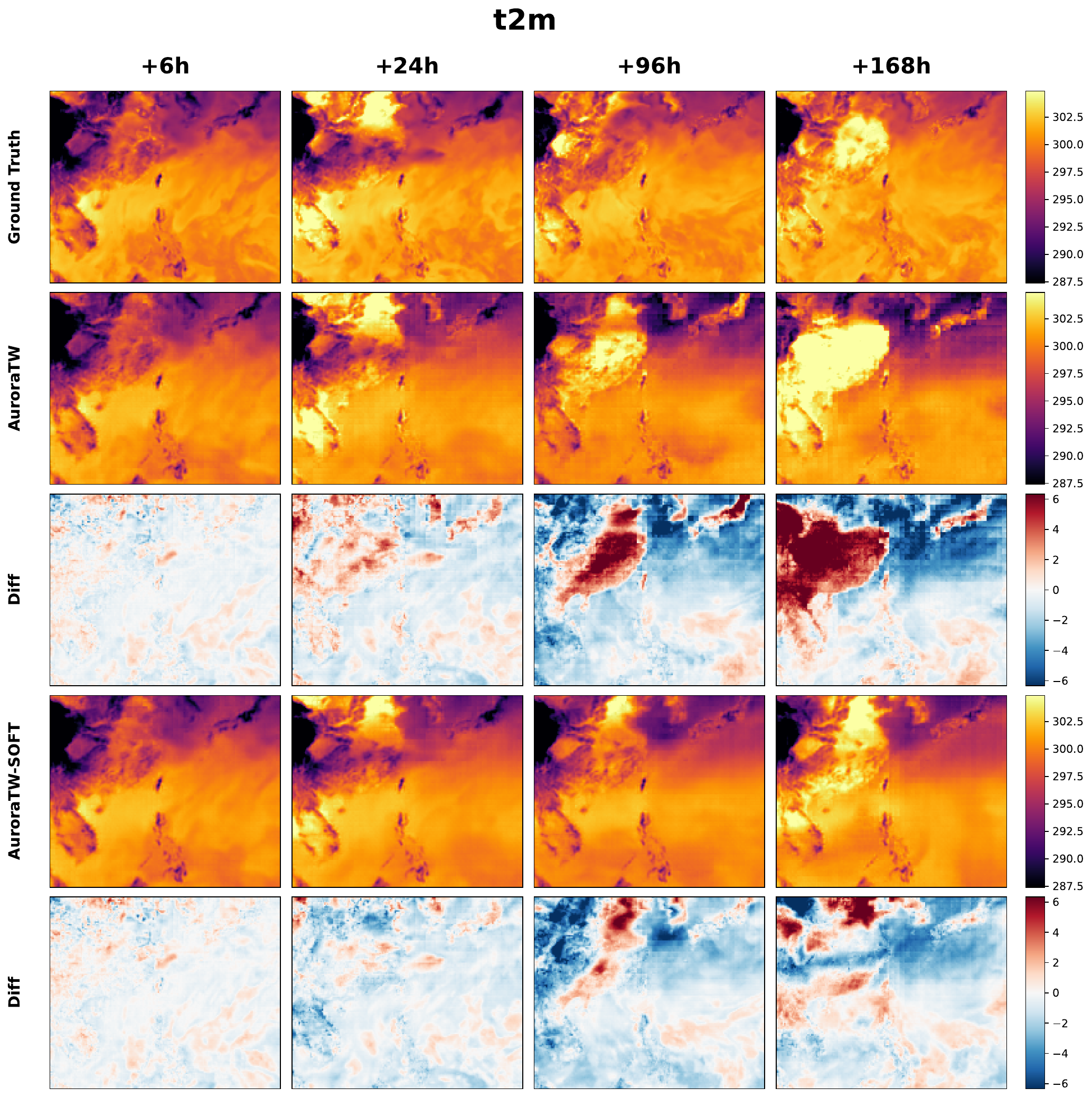}
    \caption{Visualization of t2m predictions at different lead times.}
    \label{fig:vis_app_t2m}
    \vspace*{\fill}
    \Description{t2m}
\end{figure*}

\begin{figure*}[p]
    \centering
    \vspace*{\fill}
    \includegraphics[width=0.95\linewidth]{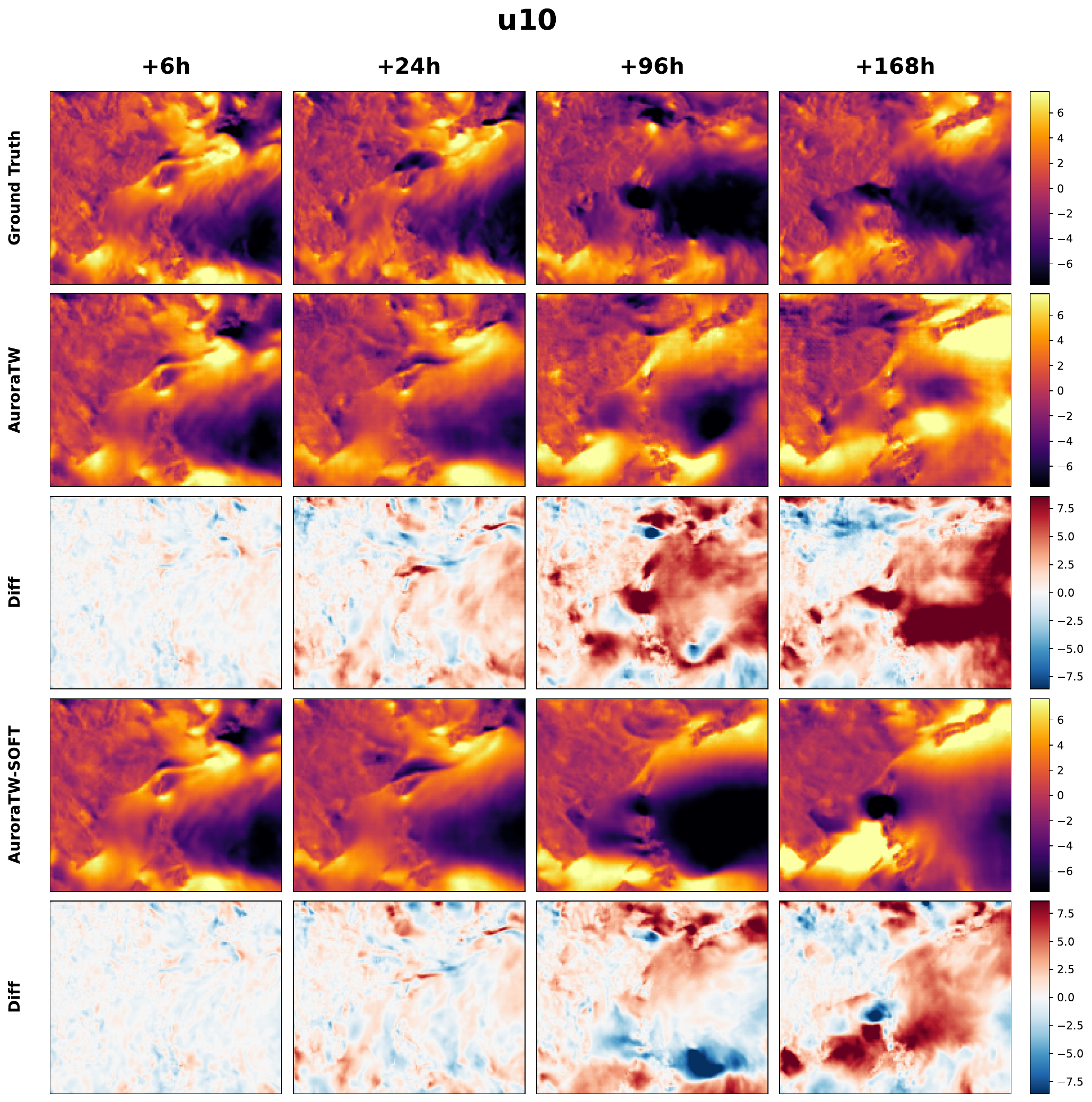}
    \caption{Visualization of u10 predictions at different lead times.}
    \label{fig:vis_app_u10}
    \vspace*{\fill}
    \Description{u10}
\end{figure*}

\begin{figure*}[p]
    \centering
    \vspace*{\fill}
    \includegraphics[width=0.95\linewidth]{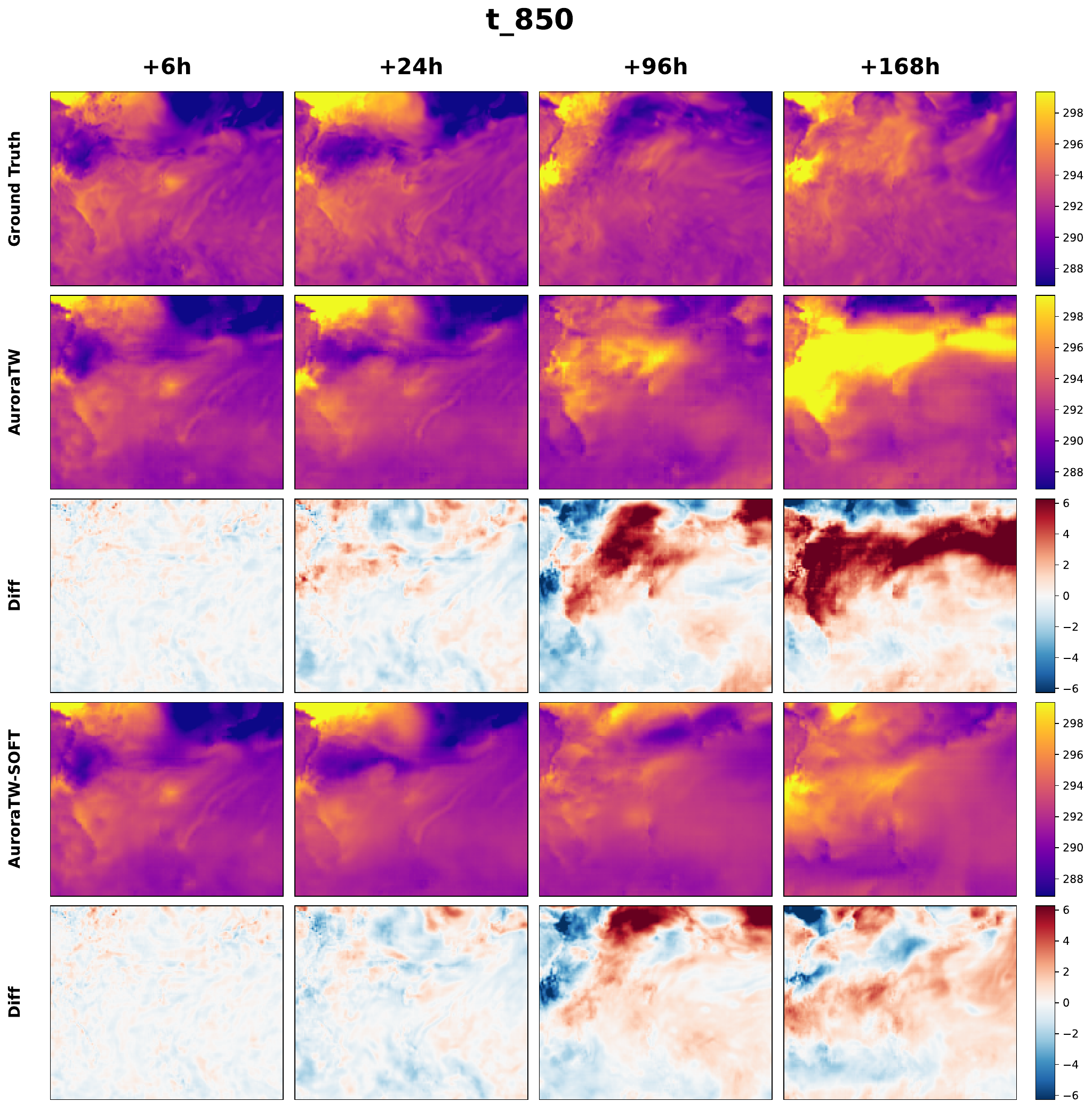}
    \caption{Visualization of t\_850 predictions at different lead times.}
    \label{fig:vis_app_t_850}
    \vspace*{\fill}
    \Description{t\_850}
\end{figure*}

\begin{figure*}[p]
    \centering
    \vspace*{\fill}
    \includegraphics[width=0.95\linewidth]{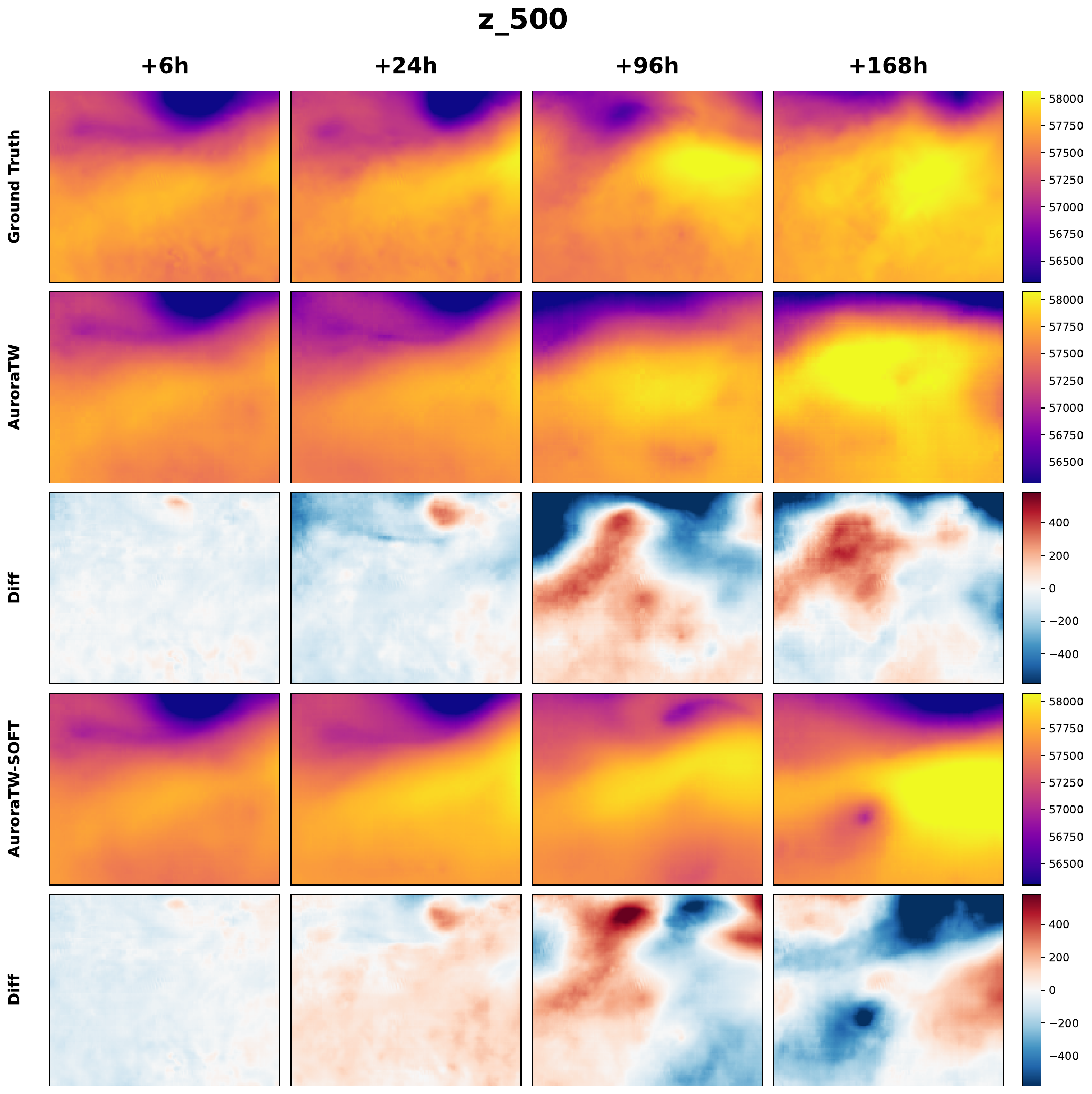}
    \caption{Visualization of z\_500 predictions at different lead times.}
    \label{fig:vis_app_z_500}
    \vspace*{\fill}
    \Description{z\_500}
\end{figure*}

\end{document}